\documentclass[conference,a4paper]{IEEEtran}
\IEEEoverridecommandlockouts
\usepackage[utf8]{inputenc} 
\usepackage[T1]{fontenc}
\usepackage{url}              
\usepackage{cite}             

\usepackage[cmex10]{amsmath}  
\usepackage{mleftright}       
\mleftright                   

\usepackage{graphicx}         
\usepackage{booktabs}         

\usepackage{amsmath,amssymb,amsthm,amsfonts}
\usepackage{dsfont}
\theoremstyle{plain}
\newtheorem{assumption}{Assumption}
\newtheorem{theorem}{Theorem}
\newtheorem{lemma}{Lemma}
\newtheorem{corollary}{Corollary}

\newtheorem{example}{Example}
\theoremstyle{remark}
\newtheorem{remark}{Remark}

\newcommand{\Ac}{\mathcal{A}}
\newcommand{\Bc}{\mathcal{B}}

\newcommand{\Dc}{\mathcal{D}}

\newcommand{\Gc}{\mathcal{G}}
\newcommand{\Hc}{\mathcal{H}}

\newcommand{\Mc}{\mathcal{M}}

\newcommand{\Xc}{\mathcal{X}}
\newcommand{\Yc}{\mathcal{Y}}
\newcommand{\Zc}{\mathcal{Z}}

\newcommand{\bb}{\boldsymbol{b}}

\newcommand{\br}{\boldsymbol{r}}

\newcommand{\by}{\boldsymbol{y}}

\newcommand{\Ab}{\mathbb{A}}
\newcommand{\Bb}{\mathbb{B}}

\newcommand{\Eb}{\mathbb{E}}

\newcommand{\Pb}{\mathbb{P}}

\newcommand{\Rb}{\mathbb{R}}

\newcommand{\vect}[1]{\text{vec}(#1)}
\newcommand{\tr}[1]{\text{tr}\left(#1\right)}
\newcommand{\sign}[1]{\text{sign}(#1)}
\DeclareMathOperator*{\argmax}{arg\,max}
\DeclareMathOperator*{\argmin}{arg\,min}

\usepackage{amsmath,amssymb,amsthm}
\usepackage{graphicx,bm}
\usepackage{hyperref}
\usepackage{algorithm}
\usepackage{algpseudocode}
\usepackage{makecell}
\usepackage{multirow}
\usepackage{xfrac}
\usepackage{xcolor}
\usepackage{colortbl}

\allowdisplaybreaks

\newcommand{\alg}{{\normalfont\texttt{TOFU}} }
\newcommand{\algg}{{\normalfont\texttt{TOFU}}}

\begin{document}

\title{On High-dimensional and Low-rank Tensor Bandits\thanks{ The work of CSs was supported in part by the US National Science Foundation (NSF) under awards ECCS-2029978, ECCS-2143559, CNS-2002902, the Virginia Commonwealth Cyber Initiative (CCI) smart cities project, and the Bloomberg Data Science Ph.D. Fellowship. The work of NS was partially supported by NSF IIS-1908070.}} 

\author{%
}

%
\author{%
  \IEEEauthorblockN{Chengshuai Shi, Cong Shen, and Nicholas D. Sidiropoulos}
  \IEEEauthorblockA{
                    University of Virginia\\
                    Charlottesville, VA 22904, USA\\
                    \{cs7ync, cong, nikos\}@virginia.edu
                    }
                  }

\maketitle

\begin{abstract}
Most existing studies on linear bandits focus on a one-dimensional characterization of the overall system. While being representative, this formulation may fail to model applications with high-dimensional but favorable structures, such as the low-rank tensor representation for recommender systems. To address this limitation, this work studies a general tensor bandits model, where actions and system parameters are represented by tensors as opposed to vectors, and we particularly focus on the case that the unknown system tensor is low-rank. A novel bandit algorithm, coined \alg (Tensor Optimism in the Face of Uncertainty), is developed. \alg first leverages flexible tensor regression techniques to estimate low-dimensional subspaces associated with the system tensor. These estimates are then utilized to convert the original problem to a new one with norm constraints on its system parameters. Lastly, a norm-constrained bandit subroutine is adopted by \algg, which utilizes these constraints to avoid exploring the entire high-dimensional parameter space. Theoretical analyses show that \alg improves the best-known regret upper bound by a multiplicative factor that grows exponentially in the system order. A novel performance lower bound is also established, which further corroborates the efficiency of \algg.
\end{abstract}

\section{Introduction}
The multi-armed bandits (MAB) framework \cite{lattimore2020bandit, bubeck2012regret} has attracted growing interest in recent years as it can characterize a broad range of applications requiring sequential decision-making. An active research area in MAB is linear bandits \cite{abbasi2011improved,chu2011contextual}, where the actions are characterized by feature vectors. While being representative, this one-dimensional (i.e., vectorized) formulation may fail to capture practical applications with high-dimensional but favorable structures. We use the recommender system model to illustrate this limitation. An online shopping platform needs an effective advertising mechanism for its products. However, instead of only deciding which item to promote (as typically considered in standard linear bandits studies), the marketer also needs to consider many other factors. For example, the marketer may plan where to place to promotion (e.g., on the sidebar or as a pop-up) and how to highlight the promotion (e.g., emphasizing the discounts or the product quality). The overall strategy with all these factors will determine the effectiveness of this promotion. 

Traditional recommendation strategies often leverage tensor formulations to capture the joint decisions concerning many associated factors \cite{frolov2017tensor,symeonidis2016matrix,papalexakis2016tensors}. However, as mentioned, existing bandits strategies are largely restricted to vectorized systems. Although vectorizing multi-dimensional systems can preserve element-wise information, structural information is often lost. Especially, as recognized in \cite{frolov2017tensor,symeonidis2016matrix,papalexakis2016tensors}, tensors formulated to characterize recommender systems often process the attractive property of \emph{low-rankness} which, however, no longer exists in the vectorized systems and thus cannot be exploited.

In this work, we study a general problem of tensor bandits for online decision-making, which extends the standard one-dimensional setting of linear bandits to a multi-dimensional and multi-linear one. In particular, each action is represented by a tensor (as opposed to a vector), and the mean reward of playing an action is the inner product between its feature tensor and an unknown system tensor. Then, motivated by various practical problems, a low-rank assumption is imposed on the system tensor, and this work aims at leveraging the low-rank knowledge to facilitate bandit learning. The main contributions are summarized in the following.

$\bullet$ The studied tensor bandits framework is general in the sense that it does not have restrictions on the system dimension and the action structure, which contributes to the generalization of linear bandits and extends the applicability of the MAB study; see Appendix~\ref{sec:related} for related works.

$\bullet$ A novel learning algorithm, \alg (Tensor Optimism in the Face of Uncertainty), is proposed for the challenging problem of low-rank tensor bandits. \alg adopts flexible designs of tensor regressions to estimate low-dimensional subspaces associated with the unknown system tensor. Then, these estimates are utilized to convert the original problem into a new one, where the low-rank property is transformed into the knowledge of norm constraints on the system parameters. \alg finally adopts the LowOFUL subroutine \cite{jun2019bilinear} to incorporate these norm constraints in bandit learning to avoid exploring the entire high-dimensional parameter space. 

\begin{table}
\caption{Related Works and Regret Comparisons}
\vspace{-0.2in}
\label{tbl:summary}
\begin{center}
\begin{tabular}{c | c } 
 \hline
 \textbf{Algorithm} &  \textbf{Regret} \\ 
 \hline
 Vectorized LinUCB \cite{abbasi2011improved} & $\tilde{O}(d^N\sqrt{T})$\\
  Matricized ESTT/ESTS \cite{kangefficient} & $\tilde{O}(d^{\lceil \frac{N}{2}\rceil}r^{\lfloor \frac{N}{2}\rfloor}\sqrt{T})$\\
  Tensor Elim. \cite{zhou2020stochastic}; modified to general actions & $\tilde{O}(d^{N-1}r\sqrt{T})$\\
  \rowcolor[gray]{0.9}
  \alg (Corollary~\ref{col:upper}) & $\tilde{O}(d^2r^{N-2}\sqrt{T})$\\
  \rowcolor[gray]{0.9}
 Lower bound (Theorem~\ref{thm:lower}) &  $\Omega(r^N\sqrt{T})$\\
 \hline
\end{tabular}
\end{center}
\begin{center}
{ \scriptsize   The time horizon is $T$. The considered system tensor is order-$N$ and of size $(d, d, \cdots, d)$. It also has a multi-linear rank $(r, r, \cdots, r)$, where $r\leq d$.}
\end{center}
\vspace{-0.3in}
\end{table}

$\bullet$ Theoretical analyses demonstrate the effectiveness and efficiency of \alg with performance guarantees. In particular, the regret of \alg improves the best-known regret upper bound by a multiplicative factor of order $O((d/r)^{\lceil N/2\rceil-2})$, where $N$ is the order of the considered system tensor, $d$ is the length of its modes, and $r\leq d$ denotes its multi-linear rank. Note that this improvement becomes more significant in high-dimensional problems, i.e., growing exponentially w.r.t. $N$. A novel regret lower bound is further established, and \alg is shown to be sub-optimal only up to a factor of $O((d/r)^2)$, which does not scale with $N$. The baselines and the main results are summarized in Table~\ref{tbl:summary}.

\section{Problem Formulation}\label{sec:formulation}
\subsection{Preliminaries on Tensors}
An order-$N$ tensor $\Yc \in \Rb^{d_1\times d_2\times \cdots \times d_N}$ has $\prod_{n\in [N]}d_n$ elements and can be viewed as a hyper-rectangle with edges (referred to as modes) of lengths $(d_1, d_2, \cdots, d_N)$ (see \cite{kolda2009tensor, sidiropoulos2017tensor} for comprehensive reviews). The tensor elements are identified to by their indices along each mode, e.g., $\Yc_{i_1,i_2,\cdots,i_N}$ denotes the $(i_1, i_2, \cdots, i_N)$-th element of $\Yc$, while a block is denoted by the index set of its contained elements, e.g., the block $\Yc_{I_1, I_2, \cdots, I_N}$ represents the elements with indices $(i_1, i_2, \cdots, i_N) \in I_1\times I_2\times \cdots \times I_N$. Moreover, fibers are one-dimensional sections of a tensor (as rows and columns in a matrix); thus an order-$N$ tensor has $N$ types of fibers.

\noindent \textbf{Tensor operations.} The inner product between tensor $\Yc$ and a same-shape tensor $\Bc \in \Rb^{d_1\times d_2\times \cdots \times d_N}$ is the sum of the products of their elements:
\begin{equation*}
    \langle \Bc, \Yc \rangle = \sum_{i_1 \in [d_1]} \sum_{i_2 \in [d_2]} \cdots \sum_{i_N \in [d_N]} \Bc_{i_1,i_2,\cdots, i_N}\Yc_{i_1,i_2,\cdots, i_N}.
\end{equation*}
The Frobenius norm is then defined as $\|\Yc\|_F := \sqrt{\langle \Yc, \Yc\rangle}$.

The mode-$n$ (matrix) product $\Yc\times_n B$ between tensor $\Yc$ and matrix $B\in \Rb^{d'_n\times d_n}$ outputs an order-$N$ tensor of size $(d_1, \cdots, d_{n-1}, d'_{n}, d_{n+1}, \cdots, d_N)$ with elements:
\begin{equation*}
    \left(\Yc\times_n B\right)_{i_1,\cdots, i_{n-1}, i'_n, i_{n+1}, \cdots, i_N} = \sum_{i_n\in [d_n]} B_{i'_n, i_n}\Yc_{i_1,\cdots, i_n, \cdots,i_N}.
\end{equation*}

In addition, matricization is the process of reordering tensor elements into a matrix. The mode-$n$ matricization of tensor $\Yc$ is denoted as $\Mc_n(\Yc)$, whose columns are mode-$n$ fibers of tensor $\Yc$ and dimensions are $(d_n, \prod_{n'\in [N]/\{n\}} d_{n'})$. Similarly, vectorization converts a tensor to a vector with all its elements, which is denoted as $\vect{\Yc}$ for tensor $\Yc$.

\noindent\textbf{Tucker decomposition.} Similarly to matrices, tensor decomposition is a useful tool to characterize the structure of tensors. In this work, we mainly focus on the Tucker decomposition illustrated as follows: for tensor $\Yc$, with $r_n$ denoting the rank of its mode-$n$ matricization, i.e., $r_n = \text{rank}(\Mc_n(\Yc))$, and $U_n$ the corresponding left singular vectors of $\Mc_n(\Yc)$, there exists a core tensor $\Gc \in \Rb^{r_1\times r_2\times \cdots \times r_N}$ such that
\begin{equation*}
    \Yc = \Gc \times_1 U_1 \times_2 U_2 \times_3 \cdots \times_N U_N =: \Gc \times_{n\in [N]} U_n, 
\end{equation*}
which can be denoted as $\Yc = [[\Gc; U_1, \cdots, U_N]]$, and the tuple $(r_1, \cdots, r_N)$ is called the multi-linear rank of tensor $\Yc$.

\noindent\textbf{Additional notations.} Typically, lowercase characters (e.g., $x$) stand for scalars while vectors are denoted with bold lowercase characters (e.g., $\boldsymbol{x}$). Capital characters (e.g., $X$) are used for matrices, and calligraphic capital characters (e.g., $\Xc$) for tensors.  In addition, $\|\cdot\|_2$ denotes the Euclidean norm for vectors and the spectral norm for matrices; for a vector $\by$ and a matrix $\Gamma$, we denote $\|\by\|_\Gamma := \sqrt{\by^\top \Gamma \by}$.

\subsection{Tensor Bandits}
This work considers the following multi-dimensional bandit problem. At each time step $t\in [T]$, the player has access to an action set $\Ab_t \subseteq \Rb^{d_1\times d_2\times \cdots \times d_N}$, i.e., the elements are tensors of size $(d_1, d_2, \cdots, d_N)$. She needs to select one action $\Ac_t$ from the set $\Ab_t$, and this action would bring her a reward of
\begin{equation}\label{eqn:tensor_reward}
    r_t = \langle \Ac_t, \Xc \rangle + \varepsilon_t,
\end{equation}
where $\Xc\in \Rb^{d_1\times d_2\times \cdots \times d_N}$ is an unknown tensor of system parameters and $\varepsilon_t$ is an independent $1$-sub-Gaussian noise. We further denote $\mu_{\Ac}: = \langle\Ac, \Xc \rangle$ as the expected reward of action $\Ac$ and, without loss of generality, assume that $\|\Xc\|_F\leq C$  for $C>0$ and $\max\{\|\Ac\|_F: \Ac\in \cup_{t\in [T]}\Ab_t\}\leq 1$. 

The agent's objective is to minimize her regret against the per-step optimal actions $\Ac_t^* := \argmax_{\Ac\in \Ab_t} \langle \Ac, \Xc \rangle$ \cite{lattimore2020bandit}:
\begin{equation*}
    R(T) := \sum\nolimits_{t\in [T]} \left(\langle \Ac_t^*, \Xc \rangle  - \langle \Ac_t, \Xc \rangle \right).
\end{equation*}

\subsection{The Low-rank Structure}
It is possible to view the above problem as a $\prod_{n\in [N]}d_n$-dimensional linear bandits problem by vectorizing the action tensor $\Ac_t$ and the system tensor $\Xc$, which can then be solved by known algorithms \cite{abbasi2011improved, chu2011contextual}. However, the high-dimensional structures of this system are not preserved by vectorization. Especially, one of the most commonly observed structures in real-world applications (e.g., recommender systems \cite{frolov2017tensor,symeonidis2016matrix,papalexakis2016tensors} and healthcare \cite{zhou2013tensor, li2018tucker, yaman2019low,kanatsoulis2019tensor}) is the low-rankness. We give the general multi-linear rank assumption of $\Xc$ as follows.  
\begin{assumption}\label{aspt:low_rank}
    The unknown system tensor $\Xc$ has a multi-linear rank of $(r_1, r_2, \cdots, r_N)$ and can be decomposed as $\Xc = [[\Gc; U_1, U_2, \cdots, U_N]]$. 
\end{assumption}

To simplify the notations, in the following, it is assumed that $d_1 = \cdots = d_N = d$ while $r_1 = \cdots = r_N = r$. 
In practice, the rank $r$ is often much smaller than the mode length $d$, especially for very large $d$. Hence, the following problem is at the center of this work: \emph{can bandit algorithms be designed to exploit the low-rank structure of the system tensor?} Especially, the key question is how much performance improvement we can achieve, compared with the naive regret of $\tilde{O}(d^N\sqrt{T})$\cite{abbasi2011improved} that is obtained by directly vectorizing the actions and the system.

Note that the design and analysis can be extended to the general case of $d_1 \neq \cdots \neq d_N$ and $r_1 \neq \cdots \neq r_N$ with minor notation modifications. Also,  without loss of generality, it is assumed that  $N$ is of order $O(1)$ (i.e., a constant) and $N \geq 3$.

\begin{figure*}
    \centering
    \setlength{\abovecaptionskip}{-10pt}
    \includegraphics[width=\linewidth]{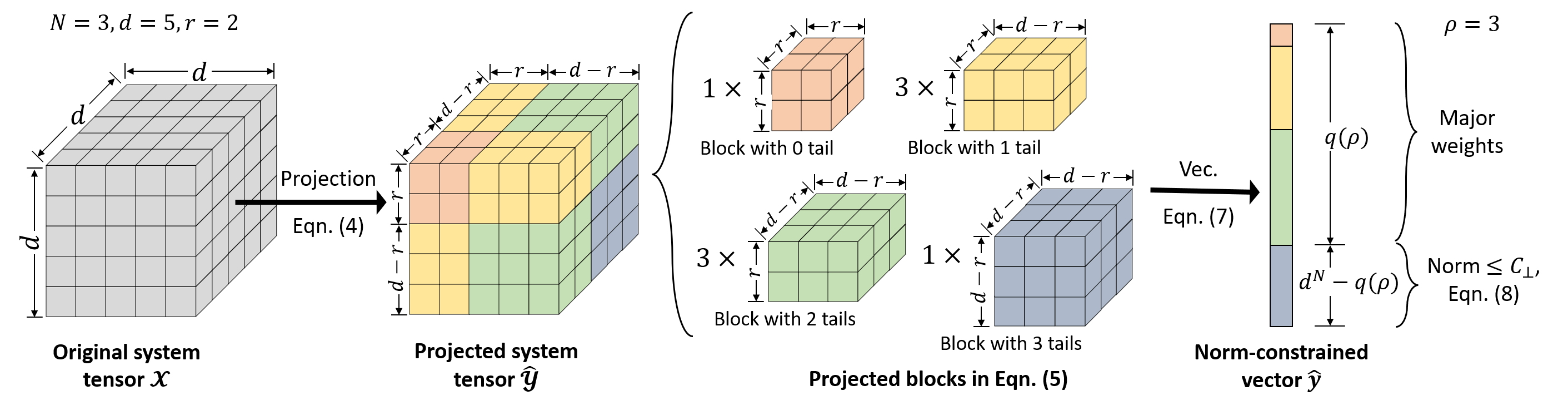}
    \caption{An illustration of the projection in Eqn.~\eqref{eqn:project_system}, the blocks with varying amounts of tails in Eqn.~\eqref{eqn:project_block}, and the vectorization in Eqn.~\eqref{eqn:vector_system}, where an order-$3$ tensor (i.e., $N=3$) is adopted as an example with $d=5$ and $ r=2$. The projection is performed with the low-dimensional subspaces estimated in Phase A (see Eqn.~\eqref{eqn:project_system}) and the projected system tensor is shown to have blocks with zero to $N$ tails. The value $q(\rho)$  is the number of elements in the projected blocks with less than $\rho$ tails as specified in Eqn.~\eqref{eqn:q}, and here an input $\rho=3$ is adopted which results in $q(3) = d^3 - (d-r)^3$. The norm constraint in Eqn.~\eqref{eqn:norm_bound} is on the other $d^N - q(\rho)$ elements, i.e., the projected blocks with at least $\rho$ tails. This constraint is leveraged in Phase B to avoid exploring the entire high-dimensional parameter space. Note that here with $N=3$ and $\rho=3$, the designed norm constraint is only on the block with three tails, while with a larger $N$, the constraint will cover more blocks if still using $\rho=3$ (as in Corollary~\ref{col:upper}), e.g., blocks with three and four tails for $N=4$.}
    \label{fig:system}
\vspace{-0.2in}
\end{figure*}

\section{The TOFU Algorithm}\label{sec:alg}
The \alg algorithm (presented in Alg.~\ref{alg}) has two phases: A and B. Phase A aims at estimating the unknown system tensor $\Xc$ up to a certain precision, especially its low-dimensional subspaces. With this estimate, the original bandit problem can be reformulated, such that the new problem has (approximately) a small number of effective system parameters because the other parameters have small norms. Then, in Phase B, an OFU (optimism in the face of uncertainty)-style subroutine is adopted to solve this norm-constrained problem. 

\subsection{Phase A: Estimating Low-dimensional Subspaces}
Phase A adopts techniques in low-rank tensor regression (also known as low-rank tensor factorization or completion from linear measurements) \cite{gandy2011tensor,jain2014provable,kanatsoulis2019tensor,zhang2020islet, ahmed2020tensor}. Especially, it considers the problem of estimating a low-rank tensor $\Xc$ by a collection of data $\{(\Ac_t, r_t): t\in [T_1]\}$ that are associated with $\Xc$ through Eqn.~\eqref{eqn:tensor_reward}, where $T_1$ is the amount of collected data samples. 

Using the bandits terminology, Phase A is designed to last $T_1$ steps, during which a dataset of $T_1$ data samples is collected. With such a dataset, an estimate of $\Xc$, denoted as $\hat{\Xc}$, can be obtained via low-rank tensor regression techniques. From another perspective, Phase A can be interpreted as using forced explorations to estimate the system tensor $\Xc$.

Clearly, the estimation quality is related to the collected data, especially the selected arms and the noises. Also, different designs of low-rank tensor regression require different data collection procedures. To provide a general discussion and ease the presentation, we denote the adopted tensor regression algorithm as $\texttt{TRalg}(\cdot)$ and consider the following assumption:
\begin{assumption}\label{aspt:phase_A}
 The dataset $\Dc_A = \{(A_t, r_t): t\in [T_1]\}$ and the tensor regression algorithm $\normalfont{\texttt{TRalg}(\cdot)}$ are such that the output $\hat{\Xc}\gets \normalfont{\texttt{TRalg}}(\Dc_A)$ satisfies $\|\hat{\Xc} - \Xc\|_F \leq \eta(T_1)$ for a problem-dependent function $\eta(T_1)$.
\end{assumption}
Under this assumption, regret bounds can be established to depend on the generic function $\eta(T_1)$. Specific dataset configurations and tensor regression algorithms can be  incorporated to establish concrete forms of $\eta(T_1)$, which leads to the corresponding problem-dependent regret bounds. Examples of datasets and algorithms that satisfy Assumption~\ref{aspt:phase_A} with a high probability can be found in Examples~\ref{exp:tensor_regression_subgaussian} and \ref{exp:tensor_regression_one_hot} in Sec.~\ref{sec:theory} with $\eta(T_1)= \tilde{O}(\sqrt{d^{N}(dr+r^N)/T_1})$. 

\subsection{From Subspace Estimates to Norm Constraints}
Intuitively, the estimated $\hat{\Xc}$ and its decomposition matrices $(\hat{U}_1, \hat{U}_2, \cdots, \hat{U}_N)$ from Phase A should help the task of bandit learning. To achieve this goal, the following projection is generalized from matrix bandits \cite{jun2019bilinear}. In particular, a new arm $\hat{\Bc}$ can be constructed from the original arm $\Ac$ as follows:
\begin{equation}\label{eqn:project_action}
    \hat{\Bc} = \Ac\times_{n\in [N]} [\hat{U}_n, \hat{U}_{n, \perp}]^\top \in \Rb^{d\times d\times \cdots \times d},
\end{equation}
where $\hat{U}_{n, \perp}$ is a set of orthogonal basis in the complementary subspace of $\hat{U}_n$ and $[\cdot, \cdot]$ denotes the concatenation of two matrices. In other words, Eqn.~\eqref{eqn:project_action} projects the actions to the estimated low-dimensional subspaces and their complements.

After some algebraic manipulations, we can establish that
\begin{align}\label{eqn:eqn}
    \mu_\Ac = \langle \Ac, \Xc \rangle = \langle \hat{\Bc}, \hat{\Yc} \rangle
\end{align}
where $\hat{\Yc}\in \Rb^{d\times d\times \cdots \times d}$ is a projected system tensor defined as
\begin{equation}\label{eqn:project_system}
    \begin{aligned}
       \hat{\Yc}&:= \Xc\times_{n\in [N]} [\hat{U}_n, \hat{U}_{n, \perp}]^\top \\
       &=\Gc\times_{n\in [N]} ([\hat{U}_n, \hat{U}_{n, \perp}]^\top U_n) .
    \end{aligned}
\end{equation}
Thus, the original tensor bandits problem can be reformulated with the action set $\hat{\Bb}_t := \{\hat{\Bc} = \Ac \times_{n\in [N]} [\hat{U}_n, \hat{U}_{n, \perp}]^\top: \Ac\in \Ab_t\}$ and the system tensor $\hat{\Yc}$ defined above. While this problem still has $d^N$ elements and the system tensor $\hat{\Yc}$ is still unknown, it possesses norm constraints on elements in many blocks of $\hat{\Yc}$. In other words, the above projection is capable of turning the low-rank property into the knowledge of parameter norms, which are specified in the following.

Especially, if $\hat{U}_n$ is estimated precisely enough, we can guarantee that $\|\hat{U}_{n,\perp}^\top U_n\|_2$ is relatively small. In particular, under Assumption~\ref{aspt:phase_A}, it holds that $\|\hat{U}_{n,\perp}^\top U_n\|_2 = \tilde{O}(\eta(T_1))$ (see Lemma~\ref{lem:tail_bound}). Then, with a closer look at the projected tensor $\hat{\Yc}$, the following observation can be made: 
 elements in many blocks are close to zero. In particular, the block
\begin{equation}\label{eqn:project_block}
\begin{aligned}
    &\hat{\Yc}_{\underbrace{:r, :r, \cdots, :r}_{\text{$N-k$ modes}}, \underbrace{r+1:,, r+1:, \cdots, r+1:}_{\text{$k$ modes}}} \\
    &= \Gc\times_{n\in [N-k]} (\hat{U}_n^\top U_n) \underbrace{\times_{n'\in [N-k+1: N]} (\hat{U}_{n',\perp}^\top U_{n'})}_{\text{with $k$ tails}},
\end{aligned}
\end{equation}
has a norm that scales with $\tilde{O}((\eta(T_1))^{k})$ (see Lemma~\ref{lem:k_tail_bound}), where the notation $:r$ denotes the set $[r]$ while $r+1:$ represents the set $[r+1:d]$ (thus the above block denotes the $r^{N-k}(d-r)^k$ tensor elements with indices $(i_1, \cdots, i_{N-k}, i_{N-k+1}, \cdots i_N)\in [r] \times \cdots \times [r] \times [r+1, d]\times \cdots \times [r+1, d]\}$). This property holds similarly for other symmetrical blocks. As $\eta(T_1)$ typically decays with $T_1$ (because the estimation quality should increase with more data samples), the norm of the above block will become smaller as the length of Phase A increases, which can be captured by a norm constraint that will be described later.

To ease the exposition, we refer to the above block and its symmetrical ones as blocks with $k$ tails, meaning the indices of their elements have $k$ modes in the interval $[r+1:d]$ (i.e., the tail). An illustration of these blocks in an order-$3$ tensor is provided in Fig.~\ref{fig:system}. Furthermore, the number of tensor elements in blocks with less than $k$ tails is denoted as
\begin{equation}\label{eqn:q}
    q(k) := \sum\nolimits_{i=0}^{k-1}{N \choose i} r^{N-i} (d-r)^{i},
\end{equation}
which is an important quantity in later designs and analyses.
\begin{remark}\label{re:one_hot}
    Compared with previous works on matrix and tensor bandits \cite{jun2019bilinear, lu2021low, kangefficient, zhou2020stochastic}, the essence of this work is the observation that norm constraints commonly exist for blocks with different numbers of tails. In particular, \cite{zhou2020stochastic} directly extends \cite{jun2019bilinear, lu2021low} and only leverages the norm constraint on the block with $N$ tails. Instead, Section~\ref{sec:theory} will illustrate that the norm constraints on blocks with at least three tails can be leveraged together under a suitable $\eta(T_1)$, which then leads to the obtained performance improvement.
\end{remark}

\begin{algorithm}[htb]
    \small
	\caption{\algg}
	\label{alg}
	\begin{algorithmic}[1]
		\Require $T$; rank $r$; dimension $N$ and $d$; tensor regression alg. $\texttt{TRalg}$; length of Phase A $T_1$; confidence parameter $\delta$; tails $\rho$
        \State Sample $\Ac_t \in \Ab_t$ following the arm selection rule required by $\texttt{TRalg}(\cdot)$ and observe reward $r_t$, for $t \in [T_1]$ \Comment{\textcolor{magenta}{\textit{Phase A}}}
        \State Estimate $\hat{\Xc}= [[\hat{\Gc}; \hat{U}_1, \cdots, \hat{U}_N]]$ with $\texttt{TRalg}$ using $\Dc_A = \{(\Ac_t, r_t): t\in [T_1]\}$, i.e., $\hat{\Xc} \gets \texttt{TRalg}(\Dc_A)$
        \State Set $C_\perp, \lambda, \lambda_\perp$ as in Theorem~\ref{thm:upper} \Comment{\textcolor{cyan}{\textit{Phase B}}}
        \State Initialize $\Lambda(\rho) \gets \text{diag}(\lambda, \cdots, \lambda, \lambda_\perp, \cdots, \lambda_\perp)$, where the first $q(\rho)$ elements are $\lambda$; $\Psi_{T_1} \gets \{\by\in \Rb^{d^N}: \|\by\|_2 \leq C\}$
        \For{$t = T_1+1, \cdots, T$}
        \State Set $\hat{\Bb}_t\gets \{\Bc_t = \Ac_t\times_{n\in [N]} [\hat{U}_n, \hat{U}_{n,\perp}]^\top: \Ac_t \in \Ab_t\}$
        \State Get $\hat{\bb}_t \gets \argmax_{\hat{\bb}_t\in \vect{\hat{\Bb}_t}}\max_{\by \in \Psi_{t-1}}\langle \hat{\bb}_t, \by \rangle$
        \State Pull arm $\Ac_t$ corresponding to $\hat{\bb}_t$ and obtain reward $r_t$
        \State Update $\hat{B}_t$ with rows $\{\bb_\tau^\top: \tau \in (T_1, t]\}$ 
        \State Update $\br_t$ with elements $\{r_\tau: \tau \in (T_1, t]\}$
        \State Update $V_t \gets \Lambda(\rho) + \hat{B}_t^\top \hat{B}_t$ and $\bar{\by} \gets  V_t^{-1} \hat{B}_t^\top \br_t$
        \State Update $\sqrt{\beta_t} \gets \sqrt{\log(\frac{\det(V_t)}{\det(\Lambda(\rho))\delta^2})} + \sqrt{\lambda }C + \sqrt{\lambda_\perp} C_\perp$  
        \State Update $\Psi_t\gets \{\hat{\by} \in \Rb^{d^N}: \|\hat{\by} - \bar{\by}\|_{V_t} \leq \sqrt{\beta_t}\}$ 
        \EndFor
	\end{algorithmic}
\end{algorithm}

\subsection{Phase B: Solving the Norm-constrained Linear Bandits}
As illustrated above, after the projection, norm constraints can be obtained on some blocks of tensor $\hat{\Yc}$. For flexibility, we consider that Phase B aims to leverage such constraints on blocks with at least $\rho$ tails, which contain $d^N-q(\rho)$ elements. The parameter $\rho$ is an input with its value in $[N]$ that requires careful designs to balance losses from two phases and will be specified in Sec.~\ref{sec:theory} (e.g., selected as $\rho = 3$ in Corollary~\ref{col:upper}). Equivalently, there exist norm constraints on parts of the elements in the unknown vector
\begin{equation}\label{eqn:vector_system}
    \hat{\by} := \vect{\hat{\Yc}} \in \Rb^{d^N}.
\end{equation}
If the vectorization of $\hat{\Yc}$ is performed first on the block with zero tail and then gradually on those with one and more tails (see Fig.~\ref{fig:system} for an example), we can compactly express the norm constraint on blocks with at least $\rho$ tails as
\begin{equation}\label{eqn:norm_bound}
    \|\hat{\by}_{q(\rho)+1: d^N}\|_2 \leq C_\perp,
\end{equation}
where the parameter $C_\perp$ will be specified later in Theorem~\ref{thm:upper}. This condition can be interpreted as that there are approximately only $q(\rho)$ effective parameters in $\hat{\by}$ while the other parameters are nearly ignorable due to their constrained norm. 

Then, a \emph{norm-constrained linear bandits} problem with $d^N$ parameters needs to be solved. In particular, the action set is $\Phi_t := \vect{\hat{\Bb}_t} \subseteq \Rb^{d^N}$ at step $t$, where $\vect{\hat{\Bb}_t}: = \{\vect{\hat{\Bc}}:\hat{\Bc}\in \hat{\Bb}_t)\}$, and the expected reward for action $\hat{\bb} \in \Phi_t$ is $\langle \hat{\bb}, \hat{\by} \rangle$. Additionally, an important norm constraint on $\hat{\by}$, i.e., Eqn.~\eqref{eqn:norm_bound}, is available to the learner.
Inspired by \cite{valko2014spectral}, the LowOFUL algorithm is designed in \cite{jun2019bilinear} to tackle such norm-constrained linear bandits. Especially, a weighted regularization is performed to estimate the system parameter: at time step $t$, the following estimate $\bar{\by}$ of $\hat{\by}$ is obtained as $\bar{\by} \gets \argmin\nolimits_{\by}\|\hat{B}_t \by - \br_t\|_2^2 + \|\by\|_{\Lambda(\rho)}^2 = V_t^{-1} \hat{B}_t^\top \br_t$,
where matrix $\hat{B}_t\in \Rb^{t\times d^N}$ is constructed with previous action vectors $\{\hat{\bb}_\tau: \tau\in (T_1, t]\}$ as rows, vector $\br_t\in \Rb^{t}$ has elements $\{r_\tau: \tau\in (T_1, t]\}$, matrix $\Lambda(\rho) = \text{diag}(\lambda, \cdots, \lambda, \lambda_\perp, \cdots, \lambda_\perp)$ (with $\lambda$ as the first $q(\rho)$ elements and $\lambda_\perp$ as the others), and $V_t = \Lambda(\rho) + \hat{B}_t^\top \hat{B}_t$. Then, an OFU-style arm-selection subroutine is adopted (lines 5--14 of Alg.~\ref{alg}).

\begin{remark}
    To better understand the projection performed in Eqns.~\eqref{eqn:project_action} and \eqref{eqn:project_system}, an ideal scenario is considered where the decomposition matrices $(U_1, \cdots, U_N)$ are {\it exactly} known. Then, the projected action $\hat{\Bc}$ and system parameter $\hat{\Yc}$ both match their ``exact'' versions $\Bc = \Ac\times_{n=1}^N [U_n, U_{n,\perp}]^\top$  and $\Yc = \Gc\times_{n\in [N]} ([U_n, U_{n, \perp}]^\top U_n) = \Gc\times_{n\in [N]} ([I_{r}, \boldsymbol{0}_{r\times (d-r)}]^\top )$. Although $\Yc$ has $d^N$ elements, there are only $r^N$ non-zero ones in $\Gc$. However, for $\hat{\Yc}$ projected via the imperfect estimates $(\hat{U}_1, \cdots, \hat{U}_N)$, we can only guarantee some blocks of elements have small norms instead of being exact nulls as in $\Yc$.
\end{remark}

\section{Theoretical Analysis}\label{sec:theory}
In this section, we formally establish the theoretical guarantee of the \alg algorithm. First, the following assumption is adopted on the minimum singular value of the matricized system tensor, which is commonly used in the study of matrix bandits \cite{jun2019bilinear,lu2021low, kangefficient} and tensor bandits \cite{han2022optimal}.
\begin{assumption}\label{aspt:min_singular_value}
    It holds that $\min_{n\in [N]}\{\omega_{\min}(\Mc_n(\Xc))\} \geq \omega$ for some parameter $\omega>0$, where $\omega_{\min}(\cdot)$ returns the minimum positive singular value of a matrix.
\end{assumption}
Then, the following regret upper bound can be established.
\begin{theorem}\label{thm:upper}
    Under Assumptions~\ref{aspt:low_rank}, \ref{aspt:phase_A} and \ref{aspt:min_singular_value}, with probability at least $1-\delta$, using $\rho\in [N]$ as input and $\lambda = C^{-2}$, $\lambda_\perp =  \frac{T}{q(\rho)\log(1+T/\lambda)}$, $C_{\perp} = 2^{N/2}C(\eta(T_1))^{\rho} \omega^{-\rho}$,
    if $T_1$ is chosen such that $\eta(T_1)\leq \omega$, the regret of \alg can be bounded as
    \begin{equation*}
        R(T) \leq \tilde{O}\left(CT_1 + d^{\rho-1} r^{N-\rho+1} \sqrt{T}+ C(\eta(T_1))^{\rho}{\omega^{-\rho}}T\right).
    \end{equation*}
\end{theorem}
It is worth noting that this theorem applies to any tensor regression technique satisfying Assumption~\ref{aspt:phase_A} and any input $\rho$, which demonstrates the flexibility of \algg. Furthermore, the above regret bound has three terms. The first term characterizes the dataset collection in Phase A. The second term represents the learning loss from the $q(\rho)$ major elements in Phase B. The third one is from the other $d^N-q(\rho)$ elements, which are nearly ignorable but still contribute to the regret.

According to function $\eta(T_1)$, parameters $\rho$ and $T_1$ should be carefully selected such that the overall regret in Theorem~\ref{thm:upper} is minimized. Two specific tensor regression techniques from \cite{han2022optimal, xia2021statistically} are considered to instantiate $\eta(T_1)$: the first one is established with the selected arms having sub-Gaussian elements, while the second selects random one-hot tensors as arms. To avoid complicated expressions, confidence parameters $\delta_1$, $\delta_2$, threshold parameters $\iota_1, \iota_2$ and scale parameters $c_1, c_2$ are adopted in the following, whose values are independent of $T_1$ and can be found in the corresponding references.

\begin{example}[Section 4.2 of \cite{han2022optimal}] \label{exp:tensor_regression_subgaussian} 
If $T_1 >\iota_1$, all elements of $\Ac_t$ are i.i.d. drawn from $1/d^N$-sub-Gaussian distributions, and $\varepsilon_t$ is an independent standard Gaussian noise, with probability at least $1-\delta_1$, an estimate $\hat{\Xc} = [[\hat{\Gc}; U_1, \cdots, U_N]]$ can be obtained from the tensor regression algorithm proposed in \cite{han2022optimal} such that $\|\hat{\Xc} - \Xc\|^2_F \leq c_1 d^N(dr+r^N)/T_1$.
\end{example}

\begin{example}[Corollary 2 of \cite{xia2021statistically}]
\label{exp:tensor_regression_one_hot}
If $T_1> \iota_2$, $\Ac_t$ is a random one-hot tensor, and $\varepsilon_t$ is an independent $1$-sub-Gaussian noise, with probability at least $1-\delta_2$, an estimate $\hat{\Xc} = [[\hat{\Gc}; U_1, \cdots, U_N]]$ can be obtained from the tensor regression algorithm proposed in \cite{xia2021statistically} such that $\|\hat{\Xc} - \Xc\|^2_F \leq c_2 d^{N}(dr+r^N)/T_1$.
\end{example}

In these examples, it can be seen that Assumption~\ref{aspt:phase_A} holds with a high probability for $\eta(T_1) = \tilde{O}(\sqrt{d^{N}(dr+r^N)/T_1})$. Then, Theorem~\ref{thm:upper} leads to the following corollary. 
\begin{corollary}\label{col:upper}
    Under Assumptions~\ref{aspt:low_rank} and \ref{aspt:min_singular_value}, if the conditions in Example~\ref{exp:tensor_regression_subgaussian} (resp. Example~\ref{exp:tensor_regression_one_hot}) can be satisfied in Phase A, using the tensor regression algorithm from \cite{han2022optimal} (resp. \cite{xia2021statistically}) as $\normalfont\texttt{TRalg}(\cdot)$, the parameters from Theorem~\ref{thm:upper} with input $\rho = 3$, and the following length for Phase A (resp. with $\iota_2, c_2$)
    \begin{equation*}
        T_1 = \max\big\{\iota_1 , c_1 d^{N}(dr+r^N){\omega^{-2}},  c_1^{\frac{3}{5}}d^{\frac{3N}{5}}(dr+r^N)^{\frac{3}{5}}{\omega^{-\frac{6}{5}}}T^{\frac{2}{5}}\big\},
    \end{equation*}
    with probability at least $1-\delta - \delta_1$ (resp. $1-\delta - \delta_2$), the regret of \alg can be bounded as 
    \begin{equation*}
        R(T) \leq \tilde{O}\left(CT_1 + d^2r^{N-2}\sqrt{T}\right)
    \end{equation*}
\end{corollary}
The above corollary adopts $\rho=3$, i.e., the norm constraint in Eqn.~\eqref{eqn:norm_bound} is on blocks with at least three tails. This choice is conscious with respect to the function $\eta(T_1)$ from Examples~\ref{exp:tensor_regression_subgaussian} and \ref{exp:tensor_regression_one_hot} as it lays aside as many parameters as possible without letting them negatively impact the bandit learning. In particular, with this choice, the length $T_1$ can be optimized as in Corollary~\ref{col:upper} (which is of order $O(T^{2/5})$) and thus the dominating term (regarding the $T$-dependency) of the regret in Corollary~\ref{col:upper} is the last one of order $\tilde{O}(d^2r^{N-2}\sqrt{T})$. 

This obtained regret of order $\tilde{O}(d^2r^{N-2}\sqrt{T})$ is compared with several existing results in the following (see also Table~\ref{tbl:summary}). First, if directly adopting linear bandits algorithms such as LinUCB \cite{abbasi2011improved} on the vectorized system, a regret of order $\tilde{O}(d^N\sqrt{T})$ would incur as the low-rank structure is not used. A second approach is to matricize the system and adopt algorithms for matrix bandits \cite{jun2019bilinear, lu2021low, kangefficient}. The state-of-the-art ESTT/ESTS \cite{kangefficient} can then achieve a regret of order $\tilde{O}(d^{\lceil \frac{N}{2}\rceil}r^{\lfloor \frac{N}{2}\rfloor}\sqrt{T})$ (see Appendix~\ref{sec:matricize}), which is still inefficient as matricization does not preserve all the structure information. At last, for \cite{zhou2020stochastic} on tensor bandits, if we modify it to have general (instead of one-hot) tensors as actions, a regret of order $\tilde{O}(d^{N-1}r\sqrt{T})$ occurs as it does not fully consider the high-dimensional benefits (see Remark~\ref{re:one_hot}). Thus, compared with the best existing regret of order $\tilde{O}(d^{\lceil \frac{N}{2}\rceil}r^{\lfloor \frac{N}{2}\rfloor}\sqrt{T})$, \alg has an improvement of a multiplicative factor of order $\tilde{O}((d/r)^{\lceil \frac{N}{2}\rceil-2})$, which grows exponentially in $N$. Hence, this benefit becomes more significant in higher-order problems.
 
While \alg improves existing results, we further compare it against the following new regret lower bound.
\begin{theorem}\label{thm:lower}
    Assume $r^N\leq 2T$ and for all $t\in [T]$, let $\Ab_t = \Ab : = \{\Ac\in \Rb^{d\times d\times \cdots \times d}: \|\Ac\|_F \leq 1\}$ and $\varepsilon_t$ be a sequence of independent standard Gaussian noise. Then, for any policy, there exists a system tensor $\Xc\in \Rb^{d\times d\times d \times \cdots \times d}$ with a multi-linear rank $(r,r,\cdots, r)$ and $\|\Xc\|_F^2 = O(r^{2N}/T)$ such that $
        \Eb_{\Xc}[R(T)] = \Omega(r^N\sqrt{T})$,
    where the expectation is taken with respect the interaction of the policy and the system.
\end{theorem}
Compared with this lower bound, \alg is sub-optimal only up to an additional $O((d/r)^2)$ factor (which does not scale with $N$). We conjecture that a slightly tighter regret lower bound of order $\Omega(dr^{N-1}\sqrt{T})$ can be established, which reduces to that of $\Omega(dr\sqrt{T})$ in matrix bandits ($N = 2$) \cite{lu2021low}. 

\section{Conclusions}
This work studied a general tensor bandits problem, where high-dimensional tensors characterize action and system parameters. Motivated by practical applications, the system tensor is modeled to be low-rank. To tackle this high-dimensional but low-rank problem, a novel algorithm named \alg  was proposed. \alg adopts tensor regression techniques to estimate low-dimensional subspaces associated with the system tensor. The obtained estimates are then used to transform the challenging problem of low-rank tensor bandits into an equivalent but easier one of norm-constrained linear bandits. The theoretical analysis provided a regret guarantee of \algg, which is shown to be exponentially more efficient than existing results. A novel performance lower bound was also established, further demonstrating the superiority of \algg.

\clearpage
\bibliographystyle{IEEEtran}
\bibliography{tensor, bandit}

\begin{thebibliography}{10}
\providecommand{\url}[1]{#1}
\csname url@samestyle\endcsname
\providecommand{\newblock}{\relax}
\providecommand{\bibinfo}[2]{#2}
\providecommand{\BIBentrySTDinterwordspacing}{\spaceskip=0pt\relax}
\providecommand{\BIBentryALTinterwordstretchfactor}{4}
\providecommand{\BIBentryALTinterwordspacing}{\spaceskip=\fontdimen2\font plus
\BIBentryALTinterwordstretchfactor\fontdimen3\font minus
  \fontdimen4\font\relax}
\providecommand{\BIBforeignlanguage}[2]{{%
\expandafter\ifx\csname l@#1\endcsname\relax
\typeout{** WARNING: IEEEtran.bst: No hyphenation pattern has been}%
\typeout{** loaded for the language `#1'. Using the pattern for}%
\typeout{** the default language instead.}%
\else
\language=\csname l@#1\endcsname
\fi
#2}}
\providecommand{\BIBdecl}{\relax}
\BIBdecl

\bibitem{lattimore2020bandit}
T.~Lattimore and C.~Szepesv{\'a}ri, \emph{Bandit algorithms}.\hskip 1em plus
  0.5em minus 0.4em\relax Cambridge University Press, 2020.

\bibitem{bubeck2012regret}
S.~Bubeck, N.~Cesa-Bianchi \emph{et~al.}, ``Regret analysis of stochastic and
  nonstochastic multi-armed bandit problems,'' \emph{Foundations and
  Trends{\textregistered} in Machine Learning}, vol.~5, no.~1, pp. 1--122,
  2012.

\bibitem{abbasi2011improved}
Y.~Abbasi-Yadkori, D.~P{\'a}l, and C.~Szepesv{\'a}ri, ``Improved algorithms for
  linear stochastic bandits,'' \emph{Advances in neural information processing
  systems}, vol.~24, 2011.

\bibitem{chu2011contextual}
W.~Chu, L.~Li, L.~Reyzin, and R.~Schapire, ``Contextual bandits with linear
  payoff functions,'' in \emph{Proceedings of the Fourteenth International
  Conference on Artificial Intelligence and Statistics}.\hskip 1em plus 0.5em
  minus 0.4em\relax JMLR Workshop and Conference Proceedings, 2011, pp.
  208--214.

\bibitem{frolov2017tensor}
E.~Frolov and I.~Oseledets, ``Tensor methods and recommender systems,''
  \emph{Wiley Interdisciplinary Reviews: Data Mining and Knowledge Discovery},
  vol.~7, no.~3, p. e1201, 2017.

\bibitem{symeonidis2016matrix}
P.~Symeonidis and A.~Zioupos, \emph{Matrix and tensor factorization techniques
  for recommender systems}.\hskip 1em plus 0.5em minus 0.4em\relax Springer,
  2016, vol.~1.

\bibitem{papalexakis2016tensors}
E.~E. Papalexakis, C.~Faloutsos, and N.~D. Sidiropoulos, ``Tensors for data
  mining and data fusion: Models, applications, and scalable algorithms,''
  \emph{ACM Transactions on Intelligent Systems and Technology (TIST)}, vol.~8,
  no.~2, pp. 1--44, 2016.

\bibitem{jun2019bilinear}
K.-S. Jun, R.~Willett, S.~Wright, and R.~Nowak, ``Bilinear bandits with
  low-rank structure,'' in \emph{International Conference on Machine
  Learning}.\hskip 1em plus 0.5em minus 0.4em\relax PMLR, 2019, pp. 3163--3172.

\bibitem{kangefficient}
Y.~Kang, C.-J. Hsieh, and T.~C.~M. Lee, ``Efficient frameworks for generalized
  low-rank matrix bandit problems,'' in \emph{Advances in Neural Information
  Processing Systems}, 2022.

\bibitem{zhou2020stochastic}
J.~Zhou, B.~Hao, Z.~Wen, J.~Zhang, and W.~W. Sun, ``Stochastic low-rank tensor
  bandits for multi-dimensional online decision making,'' \emph{arXiv
  e-prints}, pp. arXiv--2007, 2020.

\bibitem{kolda2009tensor}
T.~G. Kolda and B.~W. Bader, ``Tensor decompositions and applications,''
  \emph{SIAM review}, vol.~51, no.~3, pp. 455--500, 2009.

\bibitem{sidiropoulos2017tensor}
N.~D. Sidiropoulos, L.~De~Lathauwer, X.~Fu, K.~Huang, E.~E. Papalexakis, and
  C.~Faloutsos, ``Tensor decomposition for signal processing and machine
  learning,'' \emph{IEEE Transactions on Signal Processing}, vol.~65, no.~13,
  pp. 3551--3582, 2017.

\bibitem{zhou2013tensor}
H.~Zhou, L.~Li, and H.~Zhu, ``Tensor regression with applications in
  neuroimaging data analysis,'' \emph{Journal of the American Statistical
  Association}, vol. 108, no. 502, pp. 540--552, 2013.

\bibitem{li2018tucker}
X.~Li, D.~Xu, H.~Zhou, and L.~Li, ``Tucker tensor regression and neuroimaging
  analysis,'' \emph{Statistics in Biosciences}, vol.~10, no.~3, pp. 520--545,
  2018.

\bibitem{yaman2019low}
B.~Yaman, S.~Weing{\"a}rtner, N.~Kargas, N.~D. Sidiropoulos, and
  M.~Ak{\c{c}}akaya, ``Low-rank tensor models for improved multidimensional
  mri: Application to dynamic cardiac $ t\_1 $ mapping,'' \emph{IEEE
  transactions on computational imaging}, vol.~6, pp. 194--207, 2019.

\bibitem{kanatsoulis2019tensor}
C.~I. Kanatsoulis, X.~Fu, N.~D. Sidiropoulos, and M.~Ak{\c{c}}akaya, ``Tensor
  completion from regular sub-nyquist samples,'' \emph{IEEE Transactions on
  Signal Processing}, vol.~68, pp. 1--16, 2019.

\bibitem{gandy2011tensor}
S.~Gandy, B.~Recht, and I.~Yamada, ``Tensor completion and low-n-rank tensor
  recovery via convex optimization,'' \emph{Inverse problems}, vol.~27, no.~2,
  p. 025010, 2011.

\bibitem{jain2014provable}
P.~Jain and S.~Oh, ``Provable tensor factorization with missing data,''
  \emph{Advances in Neural Information Processing Systems}, vol.~27, 2014.

\bibitem{zhang2020islet}
A.~R. Zhang, Y.~Luo, G.~Raskutti, and M.~Yuan, ``Islet: Fast and optimal
  low-rank tensor regression via importance sketching,'' \emph{SIAM journal on
  mathematics of data science}, vol.~2, no.~2, pp. 444--479, 2020.

\bibitem{ahmed2020tensor}
T.~Ahmed, H.~Raja, and W.~U. Bajwa, ``Tensor regression using low-rank and
  sparse tucker decompositions,'' \emph{SIAM Journal on Mathematics of Data
  Science}, vol.~2, no.~4, pp. 944--966, 2020.

\bibitem{lu2021low}
Y.~Lu, A.~Meisami, and A.~Tewari, ``Low-rank generalized linear bandit
  problems,'' in \emph{International Conference on Artificial Intelligence and
  Statistics}.\hskip 1em plus 0.5em minus 0.4em\relax PMLR, 2021, pp. 460--468.

\bibitem{valko2014spectral}
M.~Valko, R.~Munos, B.~Kveton, and T.~Koc{\'a}k, ``Spectral bandits for smooth
  graph functions,'' in \emph{International Conference on Machine
  Learning}.\hskip 1em plus 0.5em minus 0.4em\relax PMLR, 2014, pp. 46--54.

\bibitem{han2022optimal}
R.~Han, R.~Willett, and A.~R. Zhang, ``An optimal statistical and computational
  framework for generalized tensor estimation,'' \emph{The Annals of
  Statistics}, vol.~50, no.~1, pp. 1--29, 2022.

\bibitem{xia2021statistically}
D.~Xia, M.~Yuan, and C.-H. Zhang, ``Statistically optimal and computationally
  efficient low rank tensor completion from noisy entries,'' \emph{The Annals
  of Statistics}, vol.~49, no.~1, 2021.

\bibitem{jang2021improved}
K.~Jang, K.-S. Jun, S.-Y. Yun, and W.~Kang, ``Improved regret bounds of
  bilinear bandits using action space analysis,'' in \emph{International
  Conference on Machine Learning}.\hskip 1em plus 0.5em minus 0.4em\relax PMLR,
  2021, pp. 4744--4754.

\bibitem{ide2022targeted}
T.~Id{\'e}, K.~Murugesan, D.~Bouneffouf, and N.~Abe, ``Targeted advertising on
  social networks using online variational tensor regression,'' \emph{arXiv
  preprint arXiv:2208.10627}, 2022.

\end{thebibliography}

\clearpage
\appendices

\section{Related Works}\label{sec:related}
\noindent\textbf{Linear bandits.} As one of the most well-studied MAB settings, linear bandits adopt vectorized features to characterize actions and system parameters.
Many provably efficient algorithms have been proposed for linear bandits and achieve (nearly) minimax optimal regret, with LinUCB being the representative design \cite{abbasi2011improved, chu2011contextual}; see \cite{lattimore2020bandit} for a comprehensive review. 

\noindent\textbf{Matrix bandits.} Several works extend the vectorized (one-dimensional) features in linear bandits to two dimensions \cite{jun2019bilinear, lu2021low, kangefficient, jang2021improved}, i.e., matrices as features. With similar motivations as this work, this line of research on ``matrix bandits'' mostly focuses on leveraging the assumed low-rank property of the system matrix, and the recent work \cite{kangefficient} achieves nearly order-optimal performance. 

\noindent\textbf{Tensor bandits.} The concept of tensor bandits is first proposed in \cite{zhou2020stochastic}, which generalizes the problem of linear bandits to high dimensions, i.e, use tensors to characterize actions and systems. However, in \cite{zhou2020stochastic}, the action tensors are restricted to be \emph{one-hot}. This work instead considers general action tensors and thus covers more scenarios. Moreover, this work has a different utilization of the estimated low-dimensional subspaces compared with \cite{zhou2020stochastic}; see details in Section~\ref{sec:alg} and Remark~\ref{re:one_hot}. Lastly, instead of the Tucker decomposition adopted in \cite{zhou2020stochastic} and this work, a recent work \cite{ide2022targeted} studies the tensor bandits problem from the Canonical polyadic decomposition (CPD) perspective, whose results are however not directly comparable to this work due to different settings.

\section{The Problem Equivalence: Derivation of Eqn.~\eqref{eqn:eqn}}
With
\begin{align*}
    \hat{\Bc}: = \Ac\times_{n\in [N]} [\hat{U}_n, \hat{U}_{n, \perp}]^\top;\\
    \hat{\Yc}:= \Xc\times_{n\in [N]} [\hat{U}_n, \hat{U}_{n, \perp}]^\top,
\end{align*}
it holds that
\begin{align*}
    &\langle \hat{\Bc}, \hat{\Yc} \rangle = \left\langle \Mc_{(1)}(\hat{\Bc}), \Mc_{(1)}(\hat{\Yc})\right\rangle\\
    & \overset{(a)}{=} \left\langle [\hat{U}_1, \hat{U}_{1,\perp}]^\top\Mc_{(1)}(\Ac) J^\top,  [\hat{U}_1, \hat{U}_{1, \perp}]^\top \Mc_{(1)}(\Xc) J^\top \right\rangle\\
    & = \tr{J\Mc_{(1)}(\Ac)^\top [\hat{U}_1, \hat{U}_{1,\perp}]  [\hat{U}_1, \hat{U}_{1, \perp}]^\top  \Mc_{(1)}(\Xc) J^\top  }\\
    & = \tr{J\Mc_{(1)}(\Ac)^\top \Mc_{(1)}(\Xc) J^\top  }\\
    & \overset{(b)}{=} \langle \Ac \times_{n=2}^N [\hat{U}_n, \hat{U}_{n, \perp}]^\top, \Xc\times_{n=2}^N [\hat{U}_n, \hat{U}_{n, \perp}]^\top\rangle\\
    & \overset{(c)}{=} \langle \Ac\times_{n=3}^N [\hat{U}_n, \hat{U}_{n, \perp}]^\top, \Xc\times_{n=3}^N [\hat{U}_n, \hat{U}_{n, \perp}]^\top\rangle\\
    & = \cdots\\
    & \overset{(d)}{=} \langle \Ac, \Xc\rangle
\end{align*}
where
\begin{align*}
    &J: = [\hat{U}_N, \hat{U}_{N, \perp}]^\top\otimes [\hat{U}_{N-1}, \hat{U}_{N-1, \perp}]^\top \otimes \cdots \otimes [\hat{U}_2, \hat{U}_{2, \perp}]^\top
\end{align*}
with $\otimes$ representing the Kronecker product between matrices. Equalities (a) and (b) use the property (see \cite{kolda2009tensor}) that
$\Zc = \Hc\times_{n\in [N]} V_n \Leftrightarrow  \Mc_{(n)}(\Zc) = V_n \Mc_{(n)}(\Hc) \left(V_{N}\otimes \cdots V_{n+1}\otimes V_{n-1}\otimes \cdots \otimes V_1\right)^\top$.
Equalities (c) and (d) recursively follow similar arguments in the previous steps. Then, following basic properties regarding tensor mode product \cite{kolda2009tensor}, it can be shown that
\begin{align*}
    \hat{\Yc}&:= \Xc\times_{n\in [N]} [\hat{U}_n, \hat{U}_{n, \perp}]^\top=  \Gc\times_{n\in [N]} \left([\hat{U}_n, \hat{U}_{n, \perp}]^\top U_n\right),
\end{align*}
which completes the derivation.

\section{Upper Bound Analysis: Proof of Theorem~\ref{thm:upper}}
\begin{lemma}\label{lem:tail_bound}
    Under Assumption~\ref{aspt:phase_A}, it holds that
    \begin{align*}
        \|\hat{U}_{n, \perp}^\top  U_n\|_F \leq \frac{\eta(T_1)}{\omega_n}, \qquad \forall n\in [N],
    \end{align*}
    where $\omega_n := \omega_{\min}(\Mc_n(\Xc))$.
\end{lemma}
\begin{proof}
    It holds that
    \begin{align*}
        \|\hat{\Xc} - \Xc\|_F &= \|\Mc_n(\hat{\Xc}) - \Mc_n(\Xc)\|_F \\
        & \overset{(a)}{=} \|\Mc_n(\hat{\Xc}) -  U_nU_n^\top \Mc_n(\Xc)\|_F\\
        & \overset{(b)}{\geq} \|\hat{U}_{n,\perp}^\top\Mc_n(\hat{\Xc}) -  \hat{U}_{n,\perp}^\top U_n U_n^\top \Mc_n(\Xc)\|_F\\
        & \overset{(c)}{=} \| \hat{U}_{n,\perp}^\top U_n U_n^\top \Mc_n(\Xc)\|_F\\
        & \overset{(d)}{\geq} \omega_{\min}(U_n^\top \Mc_n(\Xc))\|\hat{U}_{n,\perp}^\top U_n\|_F\\
        & = \omega_{n}\|\hat{U}_{n,\perp}^\top U_n\|_F,
    \end{align*}
    where equality (a) is because $U_n$ contains the left singular vectors of $\Mc_n(\Xc)$; inequality (b) is from the fact that for a matrix $U$ with orthonormal columns and an arbitrary compatible matrix $X$, it holds that $\|X\|_F \geq \|U^\top X\|_F$;
    inequality (c) uses the observation  that $\hat{U}^\top_{n,\perp}\hat{U}_n$ is a null matrix; inequality (d) is from the fact that for any compatible matrices $X$ and $Y$, $\|XY\|_F \geq \min\{\omega_{\min}(X)\|Y\|_F, \omega_{\min}(Y)\|X\|_F\}$.
    The lemma is proved by plugging Assumption~\ref{aspt:phase_A} into the above inequality.
\end{proof}

\begin{lemma}\label{lem:k_tail_bound}
    Under Assumption~\ref{aspt:phase_A}, the norm of the following block with $k$ tails in $\hat{\Yc}$ can be bounded as: 
    \begin{align*}
        \bigg\|\hat{\Yc}_{\underbrace{:r; :r, \cdots, :r}_{\text{$N-k$ modes}}, \underbrace{r+1:, r+1, \cdots, r+1:}_{\text{$k$ modes}}}\bigg\|_F \leq \frac{C (\eta(T_1))^k}{\omega^k},
    \end{align*}
    and this bound symmetrically holds for all $N \choose k$ blocks with $k$ tails.
\end{lemma}
\begin{proof}
It holds that
    \begin{align*}
        &\left\|\hat{\Yc}_{:r, :r, \cdots, :r, r+1:, r+1:, \cdots, r+1:}\right\|_F \\
        &= \left\|\Gc\times_{n\in [N-k]} (\hat{U}_n^\top U_n) \times_{n'\in [N-k+1: N]} (\hat{U}_{n',\perp}^\top U_{n'})\right\|_F\\
        &\overset{(a)}{\leq} \|\Gc\|_F \prod_{n\in [N-k]} \left\|(\hat{U}_{n}^\top U_{n})\right\|_2\prod_{n'\in [N-k+1:N]} \left\|(\hat{U}_{n',\perp}^\top U_{n'})\right\|_2\\
        &\overset{(b)}{\leq}\frac{\|\Gc\|_F \cdot (\eta(T_1))^k}{\omega^k}\leq \frac{C(\eta(T_1)^k}{\omega^k},
    \end{align*}
    where inequality (a) repeatedly uses the fact that for any arbitrary matrices $X$ and $Y$, it holds that
    $\|XY\|_F \leq \min\{\|X\|_2\|Y\|_F, \|X\|_F\|Y\|_2\}$;
    inequality (b) utilizes Lemma~\ref{lem:tail_bound}, Assumption~\ref{aspt:min_singular_value}, and the fact that for two compatible matrices $U$ and $V$ with orthonormal columns, it holds that $\|V^\top U\|_2 \leq 1$.
\end{proof}

\begin{lemma}[Corollary 1 of \cite{jun2019bilinear}]\label{lem:lowoful}
    If Eqn.~\eqref{eqn:norm_bound} holds, with
    \begin{align*}
        \lambda_\perp = \frac{T}{q(\rho)\log(1+T/\lambda)}
    \end{align*}
    the regret of LowOFUL (adopted in Phase B) for $T$ steps is, with probability at least $1-\delta$, bounded by
    \begin{equation*}
         \tilde{O}\left(\left(q(\rho)+ \sqrt{q(\rho)\lambda}C + \sqrt{T}C_\perp\right))\sqrt{T}\right).
    \end{equation*}
\end{lemma}
\begin{proof}
    Detailed proofs can be found in \cite{jun2019bilinear}.
\end{proof}

Then, the proof of Theorem~\ref{thm:upper} is presented in the following.
\begin{proof}[Proof of Theorem~\ref{thm:upper}]
    For Phase A with length $T_1$, its regret can be bounded as
    \begin{equation*}
        R_A(T)\leq 2CT_1,
    \end{equation*}
    since the mean rewards are bounded between $[-C,C]$ with $\|\Ac_t\|_F\leq 1$ and $\|\Xc\|_F\leq C$.

    After Phase A, based on Lemma~\ref{lem:k_tail_bound}, we have that for all $k\in [N]$, it holds that
    \begin{align*}
        \|\hat{\by}_{q(k)+1: q(k+1)}\|^2_F &\leq {N \choose k}\frac{C^2(\eta(T_1))^{2k}}{\omega^{2k}},
    \end{align*}
    Thus, with $q(\rho)$ in Eqn.~\eqref{eqn:q}, it can be further shown that
    \begin{align*}
        \|\hat{\by}_{q(\rho)+1: d^N}\|_F^2 &= \sum_{k\in [\rho: N]}\|\hat{\by}_{q(k)+1: q(k+1)}\|_F^2\\
        &\leq \sum_{k\in [\rho: N]}{N \choose k}\frac{C^2(\eta(T_1))^{2k}}{\omega^{2k}}\\
        &\leq \frac{2^N\cdot C^2(\eta(T_1))^{2\rho}}{\omega^{2\rho}},
    \end{align*}
    where the last inequality uses the condition that $\eta(T_1)\leq \omega$.
    This inequality validates Eqn.~\eqref{eqn:norm_bound} with the parameter $C_\perp = 2^{N/2}C(\eta(T_1))^{\rho}\omega^{-\rho}$ in Theorem~\ref{thm:upper}.

    Then, with the specified parameter, according to Lemma~\ref{lem:lowoful}, Phase B would lead to a regret bounded as 
    \begin{align*}
        R_B(T) & \leq \tilde{O}\left(\left(q(\rho)+ \sqrt{q(\rho)\lambda}C + \sqrt{T}C_\perp\right)\sqrt{T}\right)\\
        & = \tilde{O}\left(d^{\rho-1}r^{N-\rho+1} \sqrt{T}  +  \frac{ C(\eta(T_1))^{\rho}}{\omega^{\rho}}T\right),
    \end{align*}
    where the last step uses the facts that $q(\rho) = O(2^N d^{\rho-1} r^{N-\rho+1})$ and $N = O(1)$ (thus $2^N = O(1)$).
    Thus, the overall regret guarantee can be obtained as
    \begin{align*}
        &R(T) = R_A(T) + R_B(T)\\
        &\leq \tilde{O}\left(CT_1 + d^{\rho-1}r^{N-\rho+1} \sqrt{T}  +  \frac{C(\eta(T_1))^{\rho}}{\omega^{\rho}}T\right),
    \end{align*}
    which concludes the proof.
\end{proof}

\section{Upper Bound Analysis: Proof of Corollary~\ref{col:upper}}
\begin{proof}
    In the following, we prove the case for Example~\ref{exp:tensor_regression_subgaussian}. The proof for Example~\ref{exp:tensor_regression_one_hot} can be similarly constructed. First, with probability at least $1-\delta - \delta_1$, it simultaneously holds that
    \begin{align*}
        R(T) \leq \tilde{O}\left(CT_1 + d^{\rho-1}r^{N-\rho+1} \sqrt{T}  +  \frac{C(\eta(T_1))^{\rho}}{\omega^{\rho}}T\right)
    \end{align*}
    and
    \begin{align*}
        \|\hat{\Xc} - \Xc\|^2_F \leq \eta(T_1) = \sqrt{\frac{c_1 d^N(dr+r^N)}{T_1}}.
    \end{align*}
    Thus, if $\rho = 3$ as specified in Corollary~\ref{col:upper}, it holds that
    \begin{align*}
        R(T) \leq \tilde{O}\left(CT_1 + d^{2}r^{N-2} \sqrt{T}  +  \frac{ C c_1^{\frac{3}{2}}d^{\frac{3N}{2}}(dr+r^N)^{\frac{3}{2}}}{\omega^{3}T_1^{\frac{3}{2}}}T\right).
    \end{align*}
    
    With the following choice of
    \begin{align*}
        T_1 = \max\left\{\iota_1 , \frac{c_1 d^{N}(dr+r^N)}{\omega^2},  \frac{c_1^{\frac{3}{5}}d^{\frac{3N}{5}}(dr+r^N)^{\frac{3}{5}}}{\omega^{\frac{6}{5}}}T^{\frac{2}{5}}\right\},
    \end{align*}
    the threshold requirement in Example~\ref{exp:tensor_regression_subgaussian} can be satisfied (i.e., $T_1\geq \iota_1$) and it can be verified that $\eta(T_1) \leq \omega$. Thus, with probability at least $1-\delta-\delta_1$, the regret can be bounded as
    \begin{align*}
        R(T) & \leq \tilde{O}\left(CT_1 + d^2r^{N-2}\sqrt{T}+ \frac{Cc_1^\frac{3}{2}d^{\frac{3N}{2}}(dr+r^N)^{\frac{3}{2}}}{\omega^{3}T_1^{\frac{3}{2}}}T\right)\\
        &= \tilde{O}\left(CT_1 + d^2r^{N-2}\sqrt{T}\right),
    \end{align*}
    where the selected value of $T_1$ is adopted. The proof is then concluded.
\end{proof}

\section{Regret of Matricized ESTT/ESTS}\label{sec:matricize}
ESTT/ESTS \cite{kangefficient} deals with the matrix bandits problem where the actions and system parameters are characterized by matrices. In particular, when the system matrix is of size $(D_1, D_2)$ and rank $R$, a regret of $\tilde{O}((D_1 + D_2) R \sqrt{T})$ can be obtained. The straightforward way to matricize the order-$N$ system tensor considered in this work is along one mode, e.g., as $\Mc_n(X)$. This obtained system matrix $\Mc_n(X)$ would be of size $(d, d^{N-1})$ and rank $r$, which results in a regret of $\tilde{O}(d^{N-1} r \sqrt{T})$ with ESTT/ESTS. However, if we combine $\lceil N/2 \rceil$ modes in the system tensor in one matrix dimension (e.g., row), and 
 the remaining $\lfloor N/2 \rfloor$ modes in the other matrix dimension (e.g., column), a matrix of size $(d^{\lceil N/2 \rceil}, d^{\lfloor N/2 \rfloor})$ can be obtained with rank $r^{\lfloor N/2 \rfloor}$. Using this matricization, ESTT/ESTS can obtain a regret of $\tilde{O}(d^{\lceil N/2 \rceil} r^{\lfloor N/2 \rfloor} \sqrt{T})$, which is much better than $\tilde{O}(d^{N-1} r \sqrt{T})$ and thus adopted as the baseline result in the main paper.

\section{Lower Bound Analysis: Proof of Theorem~\ref{thm:lower}}
\begin{proof}[Proof of Theorem~\ref{thm:lower}]
    In the following, for $i = (i_1, i_2, \cdots, i_N)$, we adopt the simplified notation that
    \begin{align*}
        \Xc_i := \Xc_{i_1, i_2, \cdots, i_N}.
    \end{align*}
    With $\Delta := \frac{1}{8\sqrt{3}}\sqrt{\frac{r^N}{T} }$, we design $\mathfrak{G} = \{\Gc \text{ that satisfies Eqn.~\eqref{eqn:lower_bound_G}}\}\subseteq \Rb^{r\times r\times \cdots \times r}$ and $\mathfrak{X} = \{\Xc = \Gc\times_{n\in [N]}U_n: \Gc\in \mathfrak{G}\}\subseteq \Rb^{d\times d\times \cdots \times d}$, where 
    \begin{equation}\label{eqn:lower_bound_G}
         \Gc_{i}\in \{\pm \Delta\},\quad \forall i = (i_1, \cdots, i_N)\in [d] \times \cdots \times [d]
    \end{equation}
    and 
    \begin{align*}
        U_n = \begin{bmatrix}
            I_{r}\\
            0_{(d-r) \times r}
        \end{bmatrix} \in \Rb^{d\times r}, \qquad \forall n\in [N].
    \end{align*}
    It can be noted that each $\Xc \in \mathfrak{X}$ has a multi-linear rank at most $(r, r, \cdots, r)$.  For $ i = (i_1, i_2, \cdots, i_N)\in [r]\times [r]\times \cdots \times [r]$, we define
    \begin{align*}
        \tau_{i} = T\wedge \min\left\{t: \sum_{\tau\in [t]} \Ac^2_{t;i} \geq \frac{T}{r^N} \right\}.
    \end{align*}
    
    For a fixed $\Xc\in \mathfrak{X}$, we have
    \begin{align*}
        &\Eb_\Xc\left[R(T)\right] = \Eb_{\Xc}\left[\sum_{t\in [T]} \langle \Ac^* - \Ac_t, \Xc \rangle\right]\\
        &= \Delta  \Eb_{\Xc}\left[\sum_{t\in [T]}\sum_{i \in [r]\times \cdots \times [r]} \left( \frac{1}{r^{\frac{N}{2}}} - \Ac_{t; i}\cdot \sign{\Xc_{i}} \right)\right]\\
        &\geq \frac{\Delta r^{\frac{N}{2}}}{2} \Eb_{\Xc}\left[\sum_{t\in [T]}\sum_{i \in [r]\times \cdots \times [r]} \left( \frac{1}{r^{\frac{N}{2}}} - \Ac_{t; i}\cdot \sign{\Xc_{i}} \right)^2\right]\\
        &\geq \frac{\Delta r^{\frac{N}{2}}}{2} \sum_{i \in [r]\times \cdots \times [r]}  \Eb_{\Xc}\left[\sum_{t\in [\tau_{i}]}\left( \frac{1}{r^{\frac{N}{2}}} - \Ac_{t; i}\cdot  \sign{\Xc_{i}} \right)^2\right],
    \end{align*}
    where the first inequality uses the fact that $\|\Ac_t\|_F\leq 1$.

    Let $\Gc'\in \mathfrak{G}$ be another tensor such that $\Gc' = \Gc$ except $\Gc'_{i} = - \Gc_{i}$ with $i = (i_1, i_2, \cdots, i_N)$. Then, with $\Xc' = \Gc'\times_{n\in [N]} U_n$, it also holds that $\Xc' = \Xc$ except $\Xc'_{i} = - \Xc_{i}$. For $x\in \{\pm 1\}$, we define
    \begin{equation*}
        \kappa_{i}(x) = \sum_{t\in [\tau_{i}]}\left( \frac{1}{r^{\frac{N}{2}}} - \Ac_{t; i}\cdot x \right)^2.
    \end{equation*}
    Let $\Pb$ and $\Pb'$ be the distributions of $\kappa_{i}(1)$ with respect to the player interaction measure induced by $\Xc$ and $\Xc'$, respectively. Then, it holds that
    \begin{align*}
        &\Eb_{\Xc}[\kappa_{i}(1)]\overset{(a)}{\geq} \Eb_{\Xc'}[\kappa_{i}(1)] - \left(\frac{4T}{r^N} + 2\right) \sqrt{\frac{1}{2}D(\Pb, \Pb')}\\
        &\overset{(b)}{\geq } \Eb_{\Xc'}[\kappa_{i}(1)] - \left(\frac{4T}{r^N} + 2\right) \Delta  \sqrt{\Eb_{\Xc}\left[\sum_{t\in [\tau_{i}]} \Ac_{t;i}^2\right]}\\
        &\geq \Eb_{\Xc'}[\kappa_{i}(1)] - \left(\frac{4T}{r^N} + 2\right) \Delta \sqrt{\frac{T}{r^{N}}+1}\\
        &\overset{(c)}{\geq } \Eb_{\Xc'}[\kappa_{i}(1)] - \frac{8\sqrt{3}T\Delta}{r^N} \sqrt{\frac{T}{r^N} }
    \end{align*}
    where $D(\cdot, \cdot)$ is the relative entropy between two  probability measures. Inequality (a) uses the result in Exercise 14.4 of \cite{lattimore2020bandit}, the Pinsker's inequality, and the bound
    \begin{align*}
        \kappa_{i}(1) \leq 2\sum_{t\in [\tau_{i}]} \frac{1}{r^N} + 2\sum_{t\in [\tau_{i}]}  \Ac_{t; i}^2\leq \frac{4T}{r^N} + 2,
    \end{align*}
    where the definition of $\tau_i$ is used for the last inequality.
    Inequality (b) is from the chain rule for the relative entropy up to a stopping time in Exercise 15.7 of \cite{lattimore2020bandit} as follows:
    \begin{align*}
        D(\Pb, \Pb') &\leq \frac{1}{2}\Eb_{\Xc}\left[\sum_{t\in [\tau_{i}]}\left(\langle \Ac_t,  \Xc - \Xc'\rangle\right)^2\right] \\
        &= 2\Delta^2 \Eb_{\Xc}\left[\sum_{t\in [\tau_{i}]} \Ac_{t;i}^2\right].
    \end{align*}
    Inequality (c) is from the assumption that $r^N\leq 2T$.

    Then, it holds that
    \begin{align*}
         &\Eb_{\Xc}[\kappa_{i}(1)] +  \Eb_{\Xc'}[\kappa_{i}(-1)]\\ &\geq \Eb_{\Xc'}[\kappa_{i}(1) + \kappa_{i}(-1)] - \frac{8\sqrt{3}T\Delta }{r^N} \sqrt{\frac{T}{r^N} }\\
         &= 2\Eb_{\Xc'}\left[\frac{\tau_{i}}{r^N} + \sum_{t\in [\tau_{i}]}\Ac_{t;i}^2\right] - \frac{8\sqrt{3}T\Delta }{r^N}\sqrt{\frac{T}{r^N} }\\
         &\geq \frac{2T}{r^N} - \frac{8\sqrt{3}T\Delta }{r^N}\sqrt{\frac{T}{r^N} }\\
         &\geq \frac{T}{r^N},
    \end{align*}
    where the last inequality uses the definition of $\Delta$.

    The proof is completed using an averaging number argument on the following quantity:
    \begin{align*}
        &\sum_{\Xc\in \mathfrak{X}} \Eb_{\Xc}[R(T)] \geq \frac{\Delta r^{\frac{N}{2}}}{2} \sum_{i \in [r]\times \cdots \times [r]} \sum_{\Xc\in \mathfrak{X}}  \Eb_{\Xc}\left[\kappa_{i}(\sign{\Xc_{i}})\right] \\
        &= \frac{\Delta r^{\frac{N}{2}}}{2} \sum_{i \in [r]\times \cdots \times [r]}\sum_{\Xc/\Xc_i\in \{\pm\Delta\}^{r^N-1}} \sum_{\Xc_{i}\in \{\pm\Delta\}}  \Eb_{\Xc}\left[\kappa_{i}(\sign{\Xc_{i}})\right] \\
        &\geq \frac{\Delta r^{\frac{N}{2}}}{2} \sum_{i \in [r]\times \cdots \times [r]} \sum_{\Xc/\Xc_i\in \{\pm\Delta\}^{r^N-1}} \frac{T}{r^N}\\
        & = 2^{r^N-2} \Delta r^{\frac{N}{2}} T.
    \end{align*}
    Hence there exists $\Gc\in \mathfrak{G}$ and $\Xc = \Gc\times_{n\in [N]}U_n$ such that 
    \begin{align*}
        \Eb_{\Xc}[R(T)]\geq \frac{\Delta r^{\frac{N}{2}}T}{4} = \frac{ r^{N}\sqrt{T}}{32\sqrt{3}},
    \end{align*}
    which concludes the proof.
\end{proof}

\section{Discussion and Future Directions}
\noindent\textbf{Tighter upper and lower regret bounds.} As mentioned at the end of Section~\ref{sec:theory}, it is conjectured that a slightly tighter regret lower bound of order $\Omega(dr^{N-1}\sqrt{T})$ exists, which reduces to $\Omega(dr\sqrt{T})$ for matrix bandits ($N=2$) \cite{lu2021low}. It would be an interesting question to (dis)prove this conjecture, and we hope the proof of the current Theorem~\ref{thm:lower} can be inspiring. On the other hand, it remains a challenging problem to further tighten the performance upper bound established in Theorem~\ref{thm:upper} and Corollary~\ref{col:upper}. Inspired by the recent success in matrix bandits \cite{kangefficient}, one potential direction is to only estimate the subspaces $(\hat{U}_1, \cdots, \hat{U}_N)$ in Phase A, because $\hat{\Gc}$ is not used in later learning. In particular, it would be sufficient to obtain an estimate $\hat{\Xc}$ that approaches $\upsilon \Xc$ with $\upsilon$ as an unknown constant. If corresponding techniques can be proposed in the study of tensor estimation, the general framework of \alg can  be  smoothly adapted and sharper upper bounds can be similarly obtained.

\noindent\textbf{Structure action sets.} This work mainly investigates the problem with a low-rank tensor for system parameters. It would be valuable to also consider structured tensors for actions. Preliminary results for matrix bandits can be found in \cite{jang2021improved}, where low-rank action matrices are studied. 

\noindent\textbf{From Tucker to CPD.} Besides the Tucker decomposition, another well-known tensor decomposition is CPD. It would be interesting to study the problem of tensor bandits with a low-rank CPD, which might be able to eliminate the exponential dependency on $N$. However, this direction is challenging as the projections constructed in \alg cannot be performed under the CPD formulation, and requires further investigations.

\end{document}